\documentclass[twoside]{article}

\usepackage[accepted]{aistats2019}
\usepackage{amsmath}
\usepackage{amssymb}
\usepackage{amsthm}
\usepackage{mathtools}
\usepackage[ruled,vlined]{algorithm2e}
\usepackage[title]{appendix}
\usepackage{enumitem}
\usepackage{bbm}
\usepackage[round]{natbib}

\newcommand{\X}{\mathcal{X}}
\newcommand{\LL}{\mathcal{L}}
\newcommand{\Oh}{\mathcal{O}}

\newcommand{\lp}{\left (}
\newcommand{\rp}{\right )}
\newcommand{\lb}{\left [}
\newcommand{\rb}{\right ]} 
\newcommand{\mc}[1]{\mathcal{#1}}    
\newcommand{\mbb}[1]{\mathbb{#1}}

\newcommand{\xhit}{x_{h_t,i_t}}
\newcommand{\Jnh}{J_{n,h}}
\newcommand{\Jnht}{\tilde{J}_{n,h}}
\newcommand{\ght}{\tilde{\gamma}_n^{(h)}}
\newcommand{\gh}{\gamma_n^{(h)}}

\newcommand{\indi}[1]{\mathbbm{1}_{\{#1\}}}

\DeclareMathOperator*{\argmin}{arg\,min}
\DeclareMathOperator*{\argmax}{arg\,max}

\newenvironment{proofoutline}
 {\proof[Proof Outline]}
 {\endproof}

\newtheorem{theorem}{Theorem}

\newtheorem{lemma}{Lemma}

\theoremstyle{definition}
\newtheorem{definition}{Definition}
\newtheorem{remark}{Remark}
\begin{document}

%

%

\twocolumn[

\aistatstitle{Multiscale Gaussian Process Level Set Estimation }

\aistatsauthor{Shubhanshu Shekhar \And Tara Javidi}
\aistatsaddress{University of California, San Diego \\ \texttt{shshekha@eng.ucsd.edu} \And University of California, San Diego \\ \texttt{tjavidi@eng.ucsd.edu}}
]

\begin{abstract}
    In this paper, the problem of estimating the level set of a black-box function 
    from noisy and expensive evaluation queries is considered. A new algorithm for this problem
    in the Bayesian framework with a Gaussian Process (GP) prior is proposed. The 
    proposed algorithm employs a hierarchical sequence of partitions to explore different
    regions of the search space at varying levels of detail depending upon 
    their proximity to the level set boundary. It is shown that this approach 
    results in the algorithm having a low complexity implementation whose  computational cost is significantly
    smaller than the existing algorithms for higher dimensional search space $\X$.  
    Furthermore, high probability bounds on a measure of discrepancy between the 
    estimated level set and the true level set for the 
the proposed algorithm are obtained, which are shown 
    to be strictly better than the existing guarantees for a large class of GPs. 
    In the process, a tighter characterization of 
    the information gain of the proposed algorithm is obtained which takes into
    account the structured nature of the evaluation points. This approach 
    improves upon the existing technique of bounding the information gain with maximum 
    information gain.

\end{abstract}


\section{Introduction}

Suppose $f:\X \to \mathbb{R}$ is an unknown black-box function which can only be accessed through its noisy observations 
\begin{equation}
    \label{eq:observation_model}
    y = f(x) + \eta. 
\end{equation}

For some $\tau>0$, we define the $\tau$ (super-)level set of $f$ as $S_{\tau} = \{ x \in \X: f(x) \geq \tau \}$. Given a budget of $n$ function evaluations,  our goal is to design an adaptive query point selection strategy in order to efficiently construct an estimate $\hat{S}_{\tau}$ of the $\tau$ level set of $f$. 
 The accuracy of an estimate $\hat{S}_{\tau}$ is measured by the term 
 \[
     \LL(\hat{S}_{\tau}, S_{\tau}) = \sup_{x \in \hat{S}_{\tau}\triangle S_{\tau}} |f(x) - \tau|, 
 \]
 where $\hat{S}_{\tau}\Delta S_\tau = \lp \hat{S}_\tau \setminus S_\tau\rp \bigcup \lp 
 S_\tau \setminus \hat{S}_\tau\rp$ denotes the symmetric difference of the true 
 and estiamted level sets. 
 This problem of estimating level sets of unknown functions from noisy evaluations arises naturally in a wide range of applications. These applications include monitoring environmental parameters such as humidity and solar radiation~\citep{gotovos2013active}, analyzing geospatial data and medical imaging~\citep{willett2007minimax}.

 In this paper, we propose a new algorithm for level set estimation which utilizes ideas from existing algorithms in the areas of global optimization \citep{bubeck2011x, munos2011soo} and Bayesian Optimization \citep{wang2014bamsoo, shekhar2017gaussian}. 
 Compared to the state of the art, we show that our proposed algorithm has better computational complexity as well as tighter convergence guarantees.

 \subsection{Related Work}
 
\cite{bryan2005straddle} first considered the level set estimation with a GP prior and studied several heuristics for 
selecting the evaluation points based on variance, classification probability, information gain and straddle heuristic. 
They empirically compared the performance of these methods and concluded that the straddle heuristic outperformed other methods
of selecting evaluation points. 

\cite{gotovos2013active} built upon the work of \cite{bryan2005straddle} and proposed the LSE algorithm which uses a search strategy inspired by the GP-UCB algorithm of \cite{srinivas2012information}  and derived theoretical 
bounds on the convergence rate of the estimation error. 
\cite{bogunovic2016truncated} further highlighted the connection between  Bayesian Optimization and level set estimation by studying these problems in a unified framework.  
Their proposed algorithm, TruVAR,  can also deal with non-uniform observation costs and heterostedastic noise. For the case of fixed noise and cost model, their bounds match those of \cite{gotovos2013active}.

\subsection{Contributions}
 For the case of $\X = [0,1]^D$, all the algorithms mentioned above have two drawbacks: first, the computational cost of implementing them exactly increases exponentially with $D$, and second,  their theoretical convergence guarantees depend on the maximum mutual information gain $\gamma_n$. Some recent results in Bayesian Optimization literature \citep{scarlett2017lower, scarlett2018tight} suggest that bounds based on $\gamma_n$ can be quite loose, especially for the M\'atern family of kerenls.  The main contributions of this paper address these issues: 
 \vspace{-1em}
\begin{itemize}[leftmargin=*]
     \item We propose a new algorithm for level set estimation which explores the search space by employing a
         hierarchical sequence of partitions of $\X$, and show that the computational complexity of the algorithm with a given evaluation budget $n$ only has linear dependence on the dimension $D$.

     \item We also derive theoretical guarantees on the estimation error of the  proposed algorithm which improve upon the theoretical guarantees for existing algorithms.  

 \item Finally, by exploiting the structured nature of the points evaluated by our algorithm, we obtain a more refined characterization of the information gain of our algorithm. In particular, we obtain a tighter bound on the information gain for all members of the widely used M\'atern family of kernels.

\end{itemize}

 \section{Preliminaries}
 A Gaussian Process (GP) is a collection of random variables whose finite subcollections are jointly Gaussian, that is, all linear combinations of any finite subcollection are univariate Gaussian random variables. Gaussian Processes with index set $\X$ are completely specified by their mean function $\mu : \X \to \mathbb{R}$ and covariance function $k : \X \times \X \to \mathbb{R}$. We also note that a zero mean Gaussian Process with a non-degenerate covariance function $k$ induces the \emph{canonical metric} $d_k$ on the index set defined as $d_k(x_1,x_2) = \big( k(x_1,x_1) + k(x_2,x_2) - 2k(x_1,x_2) \big)^{1/2}$. 
 
 As mentioned earlier, in this paper we work under the Bayesian framework in which  we assume that the black-box function $f$ is a sample from a zero mean Gaussian Process, $GP(0,k)$, with known covariance function $k$.  
 Furthermore, we also assume that the observation noise $\eta$ is distributed as $N(0,\sigma^2)$ and the variance $\sigma^2$ is known to the algorithm. 
 Given observations $\mathcal{D}_t = \{(x_i, y_i) \mid 1 \leq i \leq t \}$, the posterior distribution at any $x \in \X$ is again a univariate Gaussian with parameters
 \begin{align*}
     \mu_t(x) & = k(x,x_{\mathcal{D}_t})J_t^{-1}y_{\mathcal{D}_t} \\
     \sigma_t^2(x) &= k(x,x) + k(x,x_{\mathcal{D}_t}) J_t^{-1}k(x_{\mathcal{D}_t},x).
\end{align*}
In the above display, $x_{\mathcal{D}_t}$ and $y_{\mathcal{D}_t}$ denote the vectors of evaluation points and their corresponding observations. The terms  $k(x, x_{\mathcal{D}_t})$ and $k(x_{\mathcal{D}_t}, x_{\mathcal{D}_t})$ denote the vector and matrix of pairwise covariance values respectively. Finally, the term $J_t$ is equal to $  \big( k(x_{\mathcal{D}_t},x_{\mathcal{D}_t}) + \sigma^2E_t\big)$ and $E_t$ is the $t \times t$ identity matrix.

Next, we introduce some definitions regarding the properties of the index space $\X$. 

\begin{definition}
    \label{def:metric_dim1}
Given a set $\X$ with an associated metric $d$, we define the \emph{metric dimension}  of $\X$ with respect to $d$, denoted by $D_m$, as follows:
\[
    D_m \coloneqq \inf \{ a>0 \mid \hspace{0.5em} \exists C<\infty : N(\X, r, d) \leq C r^{-a}\; \forall r \geq 0 \} 
\]
where $N(\X, r, d)$ is the $r-$covering number of $\X$ with respect to the metric $d$ defined as:
\[
    N(\X, r, d) \coloneqq \min \{ |\mathcal{Z}| \mid \mc{Z}\subset \X, \; \X \subset \cup_{z \in \mathcal{Z}}B(z,r,d) \}.
\]
\end{definition}
A related notion is the $r-$packing number of a set $\X$ with respect to a metric $d$, denoted by $M(\X, r, d)$,  which is defined as:
\begin{multline*}
    M(\X, r, d) \coloneqq \max \{|\mathcal{Z}| \mid \mathcal{Z} \subset \X, \\ d(z_1, z_2) \geq r \hspace{0.2em} \forall z_1,z_2 \in \mathcal{Z} \}.
\end{multline*}

Finally, we introduce a local notion of dimensionality of the metric space. 
\begin{definition}
    \label{def:local_dim}
Suppose $\mc{P}(\X)$ denotes the power set of $\X$, and $\zeta:(0,\infty)
\mapsto \mc{P}\lp \X\rp$ represents a mapping from 
    the positive real numbers to subsets of $\X$. Then we define the 
    dimension of $(\X, d)$ associated with  the mapping $\zeta(\cdot)$ as 
    \begin{multline*}
        D_{\zeta} \coloneqq \inf \{a > 0 \; \mid \exists C<\infty: \, M(\zeta(r), r, d) \leq \\
        C r^{-a} \; \forall r>0 \}.
    \end{multline*}
\end{definition}

The above definition of  dimension is a simple generalization of some existing
definitions such as the \emph{near-optimality dimension} of \citep{bubeck2011x,
munos2011soo, shekhar2017gaussian} and \emph{zooming dimension} of \citep{kleinberg2013bandits}.
For instance, the $c$-near-optimality dimension of \citep{bubeck2011x} is obtained by 
selecting $\zeta(r) = \{ x \in \X \mid \; f(x) \geq f(x^*) - c r \}$ for some $c>0$, 
where $f(x^*)$ denotes the maximum value of $f$. 


 \section{Main Results}
 We begin by stating the assumptions on the metric space $(\X,d)$ and the covariance function in Section~\ref{subsec:assumptions}, followed by a high level description of our proposed algorithm in Section~\ref{subsec:general_outline} and then present the details and the theoretical analysis in Section~\ref{subsec:algorithm_bayesian}. 
 \subsection{Assumptions}
 \label{subsec:assumptions}

 We assume that the set $\X$ is a compact metric space with associated metric $d$, and that $\X$ has a finite metric dimension, $D_m$, with respect to $d$. 
  We also assume that the metric space $(\X, d)$ admits a \emph{tree of partitions} \citep{bubeck2011x} which is a sequence of finite subsets $(\X_h)_{h\geq 0}$ of $\X$ such that 
\begin{enumerate}[font={\bfseries},label={X\arabic*}]
    \item $|\X_h| = 2^h$ and the elements of $\X_h$ are denoted by $x_{h,i}$ for $1 \leq i \leq 2^h$. 
    \item To each $x_{h,i}$ is associated a cell $\X_{h,i}$ such that $\cup_{i}\X_{h,i} = \X$ for all $h$ and $\X_{h+1, 2i-1}\cup \X_{h+1, 2i}  = \X_{h,i}$ for all $(h,i)$ pairs. 

    \item There exist constants $0 < v_2 \leq 1 \leq v_1$ and $0<\rho <1$ such that for all $(h,i)$ pairs
        \[
            B(x_{h,i},v_2\rho^h, d) \subset \X_{h,i} \subset B(x_{h,i}, v_1\rho^h, d)
        \]
\end{enumerate}

\begin{remark}
    As a concrete example, consider $\X = [0,1]^D$ for some $D>0$ and let $d$ be the Euclidean metric.  In this case, the metric dimension $D_m$ is equal to $D$, the dimension of the space. Now, let $\X_0 = (0.5,0.5,\ldots, 0.5)$ and the associated cell $\X_{0,1} = \X$. For any $h \geq 1$ the cells are constructed by dividing  cells from the level $h-1$ equally along the longest side (breaking ties arbitrarily), and the set $\X_h$ is defined as the center points of the cells so obtained. This tree of partitions satisfies the assumptions $X1-X3$ with parameters $\rho = 2^{-1/D}$, $v_1 = 2\sqrt{D}$ and $v_2 = 1/2$.  
\end{remark}

Next, we state our assumptions on the covariance function $k$ and the metric $d_k$ it induces on $\X$: 

\begin{enumerate}[font={\bfseries}, label ={C\arabic*}]
    \item There exists a non-decreasing continuous function $g:\mathbb{R}^+\rightarrow \mathbb{R}^+$, with ${g(0)=0}$, such that $d_k(x_1,x_2) \leq g(d(x_1,x_2))$ for all $x_1,x_2 \in \X$. 
    \item There exists a $\delta_k>0$ such that for all $r \leq \delta_k$, we have for constants $C_k>0$ and $0<\alpha \leq 1$ satisfying $g(r) \leq C_k r^{\alpha}$. 	
	\end{enumerate}

    \begin{remark}
        These two assumptions are satisfied by all commonly used covariance functions such as Squared Exponential (SE), M\'atern family and Rational Quadratic kernels. For example, for the case of SE kernel with scale and length parameters $a_s$ and $a_l$ respectively, we have $d_k(x_1,x_2) = \sqrt{2a_s\big(1- \exp(-d(x_1,x_2)^2/a_l)\big)} \leq \sqrt{2a_s/a_l}d(x_1,x_2)$ which implies the assumptions $C1-C2$ are satisfied for $\delta_k = diam(\X)$ and $\alpha=1$. 
\end{remark}  
   
\begin{remark}

It is easy to check that the class  of covariance functions satisfying $C1$ and $C2$, denoted by $\mathcal{K}$, is closed under finite linear combinations. Hence it includes various GP models useful in practical applications which are constructed by combining commonly used covariance functions \citep{duvenaud2014automatic}. 
\end{remark}

 \subsection{General Outline}
 \label{subsec:general_outline}
 We now present a high level outline of the proposed algorithm for level set estimation. 
 \begin{itemize}
     \item At any time $t$, we maintain an \emph{active} set of points $\X_t$, and their associated cells.

     \item For every point $x \in \X_t$ we compute bounds on the maximum and minimum function value in the associated cell. 

     \item In each iteration, we choose a candidate point $x_t$ from $\X_t$ which has the highest deviation from the 
         threshold $\tau$. 

     \item We take one of two actions:
         \begin{itemize}
             \item If the selected point has been explored enough, we \emph{refine} the cell. 
             \item Otherwise, we evaluate the function at the point $x_t$. 
         \end{itemize}
 \end{itemize}

In our proposed algorithm, the upper and lower bounds on the function value in each cell consist of two terms: an uncertainty term due to the observation noise and another term which estimates the variation of the function in the cell. When the uncertainty due to observation noise is smaller than the variation, it implies that the cell has been sufficiently explored at the current scale, and we proceed to refine it into smaller cells. On the other hand, if the uncertainty due to noise is larger than variation, it means that the cell requires more function evaluations at the current scale.

 \subsection{Algorithm for GP level set estimation}
 \label{subsec:algorithm_bayesian}
 The steps of our proposed algorithm for level set estimation with GP prior assumptions are shown in Algorithm~\ref{alg:bayesian_alg}. Besides the budget $n$, the threshold $\tau$ and the tree of partitions $(\X_h)_{h\geq 0}$,  the algorithm also requires as input several other parameters  $\beta_n$, $(V_h)_{h\geq 0}$ and $h_{max}$. The term $\beta_n$ is the scaling factor used in computing the posterior confidence intervals, and $V_h$ is a high probability upper bound on the variation of the unknwon function in any cell $\X_{h,i}$. The term $h_{max}$ denotes the largest depth that the algorithm should explore in the tree of partitions $(\X_h)_{h \geq 0}$.

 At any time $t$, the algorithm maintains two sets, $\hat{S}_t$ and $\hat{R}_t$, 
which contain points that do not require further consideration. More specifically,  
set $\hat{S}_t$ contains points whose lower bounds are greater than or equal to
 $\tau$ and thus with high probability we have $\hat{S}_t \subset S_\tau$. Similarly, 
 we also have $\hat{R}_t \subset S_\tau^c$ for all values of $t$. 
 
 \begin{algorithm}
     \label{alg:bayesian_alg}
     \caption{Level set estimation with GP prior}
 \SetAlgoLined
 \SetKwInput{Input}{Input}
 \SetKwInput{Output}{Output}
 \SetKw{Init}{Initialize}
 \Input{$n$,  $\tau$, $(\mathcal{X}_h)_{h\geq 0}$, $(V_h)_{h\geq 0}$, $\beta_n$, $h_{max}$} 
 
 \Init{$t=1$, $n_e=0$,  $\hat{S}_t = \phi$, $\hat{R}_t = \phi$}
 
 \While{$n_e \leq n$}{

 \For{$x_{h,i}\in \X_t$} 
 {
 
     \uIf{$\bar{l}_t(x_{h,i}) \geq \tau$ }{
         $\hat{S}_t$ $\leftarrow$ $\hat{S}_t \cup \X_{h,i}$\;
         $\X_t \leftarrow \X_t \setminus \{x_{h,i}\}$\;
 }
 \uElseIf{$\bar{u}_t(x_{h,i}) < \tau$}{
     $\hat{R}_t \leftarrow \hat{R}_t \cup \X_{h,i}$\;
     $\X_t \leftarrow \X_t \setminus \{ x_{h,i} \}$\;
 }
 }
 $x_{h_t,i_t} \in \argmax_{x_{h,i}\in \X_t} \max \big(\bar{u}_t(x_{h,i}) - \tau,\hspace{0.5em} \tau- \bar{l}_t(x_{h,i})\big)$\;
 \uIf{$\beta_n \sigma_{t-1}(x_{h_t,i_t}) < V_{h_t}$ AND $h_t \leq h_{max}$}{
 $\X_t \leftarrow \X_t \setminus \{x_{h_t,i_t}\}$ \;
 $\X_t \leftarrow \X_t \cup \{ x_{h_t+1, j} \mid p(x_{h_t+1, j}) = x_{h_t,i_t} \}$\;
 }
 \Else{ 
  evaluate $y_t = f(x_{h_t,i_t}) + \eta_t$\;
  update $\mu_t(\cdot)$, $\sigma_t(\cdot)$\; 
  $n_e \leftarrow n_e+1$\;
  }
$t \leftarrow t+1$ \;
}
 
 \Output{$\hat{S}_t$}
\end{algorithm}

We now complete the description of the algorithm by specifying the choice of the terms $\beta_n$, $(V_h)_{h\geq 0}$, $h_{max}$, $\bar{l}_t$ and $\bar{u}_t$. The detailed reasoning for these choices are provided in Appendix~\ref{appendix:choice}.  
    \begin{itemize}
        \item For any $\delta>0$, we select
            $\beta_n = \sqrt{ 2\log n \lp 2n^{1 + 2/(2\alpha \log(1/\bar{\rho}))}\rp + 
            2\log(1/\delta)}$, where $\bar{\rho} = \min\{\rho, 1/2\}$, and $\alpha$ 
            is the parameter introduced in Assumption~\textbf{C2}.
            This choice ensures that $\forall t \geq 1$ and $\forall x \in \X_t$, we have $|f(x) - \mu_{t-1}(x)| \leq \beta_n \sigma_{t-1}(x)$ with probability $\geq 1-\delta$. 
            
        \item For any $\delta>0$, we select
            \begin{multline*}
                V_h = g(v_1\rho^h)\bigg( \big( C_2 + 2\log(1/\delta_h) \\+ (4D_m')\log(1/v_1\rho^h) \big)^{1/2} + C_3\bigg)
        \end{multline*}
        where $C_2$ and $C_3$ are constants whose exact expressions are given in 
        Appendix~\ref{appendix:choice}, $\delta_h = \delta/(2^h h_{\max})$ where 
        $h_{\max}$ is introduced below, and $D_m' = D_m/\alpha$ where $D_m$ is the 
        metric dimension of $(\X, d)$ and $\alpha$ is the parameter introduced in 
        Assumption~\textbf{C2}. 
        With this choice of $V_h$, we have with probability at least $1-\delta$, 
            \[
        \sup_{x \in \X_{h,i}}|f(x) - f(x_{h,i}| \leq V_h \hspace{1em} \forall h,i
            \]
            The expression for $V_h$ is obtained by using classical chaining arguments\citep[\S~5.3]{van2014probability} 
            along with the assumptions on the covariance function. 
        \item We choose the value of $h_{max}$ to be $\log(n)/(2\alpha \log(1/\bar{\rho}))$
        where $\bar{\rho} = \min\{\rho, 1/2 \}$. This choice of $h_{max}$ along with the finite metric dimension assumption ensures that the size of $\X_t$ for all $t$ is at most polynomial in $n$ which allows us to construct tight confidence bounds on the function values at all points in $\X_t$.  
    \end{itemize}
    It now remains to define the terms $\bar{u}_t$ and $\bar{l}_t$. To compute the lower bound $\bar{l}_t(x_{h,i})$ on the function value in a cell $\X_{h,i}$, we first obtain a lower bound on the function value at $x_{h,i}$ and then subtract $V_h$ from it. The lower bound on $f(x_{h,i})$ is obtained by computing two lower bounds and taking the maximum. The term $\bar{u}_t$ is computed in a similar manner as well. The details of the computations are as follows:

            \begin{equation*}
                \bar{l}_t(x_{h,i}) = \max\{ \bar{l}_{t-1}(x_{h,i}), l_t(x_{h,i}) \}, \\
            \end{equation*}
            where
            \begin{multline*}
                l_t(x_{h,i}) = \max \{ \mu_t(x_{h,i}) - \beta_n \sigma_t(x_{h,i}),\\ \mu_t(p(x_{h,i})) - \beta_n\sigma_t(p(x_{h,i})) - V_{h-1} \} - V_h, 
                \end{multline*} 
and 
            \begin{equation*}
                \bar{u}_t(x_{h,i}) = \min\{ \bar{u}_{t-1}(x_{h,i}), u_t(x_{h,i}) \}, \\
            \end{equation*}
            where
            \begin{multline*}
                u_t(x_{h,i}) = \min \{ \mu_t(x_{h,i}) + \beta_n \sigma_t(x_{h,i}), \\ \mu_t(p(x_{h,i})) + \beta_n\sigma_t(p(x_{h,i})) + V_{h-1} \} + V_h. 
                \end{multline*} 
                In the  above display, for any $h \geq 1$ and $1\leq i \leq 2^h$, we use $p(x_{h,i})$
                to denote the parent node of $x_{h,i}$ in the tree of partitions, i.e., 
                $p(x_{h,i}) = \argmin_{x \in \X_{h-1}} d\lp x, x_{h,i}\rp$.

                    We now proceed to the theoretical analysis of Algorithm~\ref{alg:bayesian_alg} and begin by presenting a lemma which characterizes the properties of the points which are evaluated the algorithm.  

	\begin{lemma} 
	\label{lemma:alg1}
    For the choice of parameters described above, we have for any $\delta>0$, with probability at least $1-2\delta$:
		\begin{itemize}
        
            \item If at time $t$ a point $x_{h_t,i_t}$ is evaluated by the algorithm, then the maximum deviation  from $\tau$ of the function value in the cell $\X_{h_t,i_t}$ can be upper bounded as follows: 
       \begin{equation}
\label{eq:delta_1}
        \sup_{x \in \X_{h_t,i_t}}|f(x) - \tau| \leq 10V_{h_t}
       \end{equation}
   \item  If the evaluated point $x_{h_t,i_t}$ also  satisfies the condition that $h_t < h_{\max}$, then we  can bound the maximum devitation from $\tau$ in another way using the posterior standard deviation at $x_{h_t,i_t}$:
	\begin{equation}
\label{eq:delta_2}
\sup_{x \in \X_{h_t,i_t}}|f(x) - \tau| \leq  4\beta_n \sigma_{t}(x_{h_t,i_t})
	\end{equation}

            \item A point $x_{h,i}$, with $h < h_{\max}$, may be evaluated no more than $q_h$ times before it is expanded, where 
           \[
            q_h = \frac{\sigma^2\beta_n^2}{V_h^2}.
            \]
            and for $h$ large enough so that $v_1\rho^h \leq \delta_k$, we have 
\[
    q_h = \mathcal{O}\bigg( \frac{\sigma^2\beta_n^2}{(v_1\rho^h)^{2\alpha}}\bigg)
\]
	
		\end{itemize}
	\end{lemma}

\begin{proof}
    We prove the three statements separately. 
    \begin{itemize}
        \item 
We observe that if a point is evaluated by the algorithm, then we must have $\bar{l}_t(\xhit)
\leq \tau \leq \bar{u}_t(\xhit)$. This implies that $\max\{\bar{u}_t(\xhit)-\tau, \, 
\tau - \bar{l}_t(\xhit)\} \leq \bar{u}_t(\xhit) - \bar{l}_t(\xhit)$.
Now, using the fact that $\bar{u}_t(x_{h_t,i_t}) 
\leq \mu_t(p(x_{h_t,i_t})) + \beta_n\sigma_t(p(x_{h_t,i_t})) + V_{h_t-1} + V_{h_t}$,
and $\bar{l}_t(x_{h_t,i_t}) \geq \mu_t(p(x_{h_t,i_t})) - \beta_n\sigma_t(p(x_{h_t,i_t}))
- V_{h_t-1} - V_{h_t}$ we get for any $x \in \X_{h_t,i_t}$. 
\begin{align*}
    |f(x) - \tau| &\leq \bar{u}_t(\xhit) - \bar{l}_t(\xhit) \\
                  &\leq 2\beta_n\sigma_t(p(x_{h_t,i_t})) + 2V_{h_t-1} + 2V_{h_t}\\
                  &\stackrel{(a)}{\leq} 4V_{h_t-1} + 2V_{h_t} 
                  \stackrel{(b)}{\leq} 10V_{h_t}
\end{align*}
where $(a)$ follows from the fact that $\beta_n \sigma_t\lp p(\xhit)\rp$ must be 
smaller than $V_{h_t-1}$ for the cell associated with $p(\xhit)$ to be refined
and $(b)$ follows from the fact that $V_{h_t-1} \leq 2V_{h_t}$ (see Remark~\ref{remark:V_h} in Appendix~\ref{appendix:choice}).

\item 
Assume that a point $x_{h_t, i_t}$ is evaluated by the algorithm at time $t$. Then 
for any $x \in \X_{h_t,i_t}$ we have 
\begin{align*}
    |f(x) - \tau| &\leq \max\{ \bar{u}_t(x_{h_t,i_t})- \tau, \tau - \bar{l}_t(x_{h_t,i_t})\} \\
                  &\leq \bar{u}_t(x_{h_t,i_t}) - \bar{l}_t(x_{h_t,i_t})\\
                  &\stackrel{(a)}{\leq} 2\beta_n\sigma_t(x_{h_t,i_t}) + 2V_{h_t} \\
                  &\stackrel{(b)}{\leq } 4 \beta_n \sigma_t(x_{h_t,i_t})\\
\end{align*}
where $(a)$ follows from the fact that by definition $\bar{u}_t(x_{h_t,i_t}) \leq \mu_t(x_{h_t,i_t}) + \beta_n\sigma_t(x_{h_t,i_t}) + V_{h_t}  $ and $\bar{l}_t(x_{h_t,i_t}) \geq \mu_t(x_{h_t,i_t}) - \beta_n\sigma_t(x_{h_t,i_t}) - V_{h_t}  $

and $(b)$ follows from 
the condition for function evaluation at the point $\xhit$.

\item 
    We observe that from the first part of Proposition~3 of \citep{shekhar2017gaussian}, 
    if a point $x_{h,i}$ has been evaluated $n_{h,i}$ times by the algorithm, then 
    we must have $\sigma_t(x_{h,i}) \leq \sigma/(\sqrt{n_{h,i}}))$. Using this 
    fact, we can obtain an upper bound on the number of times the algorithm evaluates
    a point $x_{h,i}$ before refining, denoted by $q_h$, as follows:
   
\[
q_h = \min \{m: \beta_n \frac{\sigma}{\sqrt{m}} \leq V_h \}.
\]
On simplifying, we get the required result $q_h \leq \lp \sigma^2 \beta_n^2 \rp/V_h^2$. 
\end{itemize}
\end{proof}

    The first statement in the above lemma tells us that the points evaluated by the algorithm which lie deeper in the tree of partitions have smaller deviation from the threshold $\tau$, or alternatively, the algorithm discretizes the search space coarsely in the regions far from the threshold an constructs finer partitions in the regions close to the threshold. The second statement provides a bound on the deviation of the evaluated points in terms of the posterior standard deviation, and thus combined with the first statement described how the algorithm balances exploration (evaluating points with high $\sigma_{t}(\cdot)$) and exploitation. Finally, the last satement of Lemma~\ref{lemma:alg1} tells us that the algorithm evaluates more points in the deeper parts of the tree of partitions.

    Before proceeding to the convergence analysis of Algorithm~\ref{alg:bayesian_alg}, we need to introduce some definitions.
    For any $r>0$,
    define
    $h_r \coloneqq \max \{ h \geq 0: v_1\rho^h \geq r \}$. 
    Then we define the dimensionality of the region of $\X$ at which the function $f$ takes values close to $\tau$, as $\tilde{D} \coloneqq D_{\zeta}$, where $\zeta(r) = \{x \in \X \mid \; 
        |f(x)- \tau| \leq 10V_{h_r}\}$ and $D_{\zeta}$ was introduced in Definition~\ref{def:local_dim}. 

We note that by definition, the random variable $\tilde{D}$ is almost surely bounded by the metric dimension $D_m$ of the metric space $(\X,d)$. More specifically for the case of $\X \subset [0,1]^D$, we have $\tilde{D} \leq D$ almost surely.  

Finally, we introduce the term $J_n$ which is the sum of the posterior variance of the points evaluated by the algorithm, defined as
\[
    J_n \coloneqq \sum_{t \in Q_n} \sigma_t^2(x_{h_t,i_t}),
\]
where $Q_n$ is the set of times at which the algorithm performed function evaluations. 

    We can now state the main result of this section, which bounds the approximation error of our proposed algorithm in two ways with high probability. The first bound is in terms of $\tilde{D}$, while the second bound is in terms of $J_n$.

\begin{theorem}
    \label{thm:alg1}
    Assuming that $f$ is a sample from $GP(0,k)$ with $k \in \mathcal{K}$, the following two statements are true have with probability at least $1-2\delta$,
\begin{align}
    \LL(\hat{S}_{\tau}, S_{\tau}) &= \tilde{\mathcal{O}} \bigg(  n^{-\frac{\alpha}{\tilde D + 2\alpha}} \bigg) \label{eq:dim_type} \\
    \LL(\hat{S}_{\tau}, S_{\tau}) &=  \tilde{\mathcal{O}} \bigg(\beta_n \sqrt{J_n/n}\bigg)\label{eq:inf_type1}
    \end{align}
    where $\tilde{\Oh}$ suppresses the polylogarithmic factors. The term $J_n$ in (\ref{eq:inf_type1}) can be
    further upper bounded by a constant times $I(y_{\mathcal{D}_n};f_{\mathcal{D}_n})$, the mutual information between the function and the observations at the points of evaluation $\mc{D}_n$. 
\end{theorem}

\begin{proofoutline}
    The proof of this theorem combines ideas from the proofs of \citep[Theorem~1]{gotovos2013active} and from results in global optimization literature such as \citep{munos2011soo}. More specifically, by definition of the terms $\bar{u}_t(\cdot)$ and $\bar{l}_t(\cdot)$, the upper bound on the deviation of the points chosen by the algorithm,  $\bar{u}_t(x_{h_t,i_t}) - \bar{l}_t(x_{h_t,i_t})$, is monotonically non-increasing in $t$. Thus the maximum deviation from $\tau$ at any time $t$ can be upper bounded by the average of the deviations of all the points evaluated by the algorithm up to that time. Lemma~\ref{lemma:alg1} gives us two ways of bounding the maximum deviation from $\tau$ of the evaluated points, one in terms of the posterior standard deviation of the evaluated points, and another in terms of the variation $V_{h_t}$, the variation in the function value in the cell. Using the standard deviation bounds, and proceeding as in \citep{srinivas2012information, gotovos2013active}, we can obtain the bound given in (\ref{eq:inf_type1}). Finally, combining the $V_{h_t}$ based bound with our assumption on the metric space $(\X,d)$, we can obtain the dimension type bound given in (\ref{eq:dim_type}) by using counting arguments similar to those used in \citep{munos2011soo, wang2014bamsoo}.
    The details of the proof are given in Appendix~\ref{appendix:convergence} 
\end{proofoutline}

\begin{remark}
    The standard approach of obtaining explicit bounds in terms of $n$ for $I(y_{\mathcal{D}_n}; f_{\mathcal{D}_n})$, as laid out in \citep{srinivas2012information}, consists of two steps: first bound $I(y_{\mathcal{D}_n}; f_{\mathcal{D}_n})$ by $\gamma_n$, the maximum information gain with $n$ observations, defined as $\gamma_n \coloneqq \sup_{G \subset \X: |G| = n} I(\boldsymbol{y_G}; f)$, and then employ the bounds on $\gamma_n$ derived in \citep[Theorem~5]{srinivas2012information} for some commonly used covariance functions  to get the required bounds on the estimation error of the algorithm. This is also the approach followed to obtain the existing convergence guarantees for GP level set estimation \citep{gotovos2013active, bogunovic2016truncated}.  In Section~\ref{subsubsec:information_gain} we provide a more refined approach to bounding the term $J_n$ for Algorithm~\ref{alg:bayesian_alg}. 
\end{remark}

\noindent{\textbf{Low Complexity Implementation:}} The computational complexity of Algorithm~\ref{alg:bayesian_alg} in the worst case can be $\mathcal{O}(n^{\alpha \tilde{D} +3})$ which can be infeasible for large $\tilde{D}$. However, we can construct a low complexity version of Algorithm~\ref{alg:bayesian_alg} with slightly weaker theoretical guarantees by the following modifications:
\begin{itemize}
    \item Replace Line 11 in Algorithm~\ref{alg:bayesian_alg} with the following  selection rule:  $ x_{h_t,i_t} \in \argmax_{x_{h,i} \in \X_t} |\tau- \mu_t(x_{h,i})| + \beta_n \sigma_t(x_{h,i}) + V_h$
    \item  Remove the refinement rules in Lines 12-15 and refine a cell $\X_{h,i}$ if $x_{h,i}$ has been evaluated $q_h$ times, where $q_h$ is given in Lemma~\ref{lemma:alg1}. 
\end{itemize}
For this modified algorithm, it is easy to show that  we can obtain dimension-type bounds on the estimation error  given by (\ref{eq:dim_type}). However, since we do not take the posterior standard deviation into account in the selection rule, we cannot obtain the information-type bound for this algorithm. On the other hand, the size of the active set at time $t$ for this algorithm satisfies $|\X_t| \leq t$. Hence the computational cost of implementing this algorithm is dominated by the posterior calculation step which is a $\Oh(t^3)$ operation for any time $t$. Furthermore, since the cost of refining a cell is $\Oh(D)$ and there can be no more than $n$ cell refinements, the total cost of implementing this algorithm is $\Oh(n^4 +  Dn)$

\noindent{\textbf{Comparison with existing algorithms:}} Compared to the existing algorithms for level set estimation in the Bayesian framework, our algorithm has lower computational complexity as well as tighter guarantees on the estimation error.

The existing level set estimation algorithms with theoretical guarantees on their performance such \citep{gotovos2013active, bogunovic2016truncated} assume that the search space is finite. They can, however, be easily extended to continuous search spaces by selecting query points from a sequence of increasing finite subsets of the search space $\X$ as suggested by \cite{srinivas2012information}. More specifically, if $\X \subset [0,1]^D$, then the existing algorithms at any time $t$, select a query point by solving an optimization problem over a uniform grid of size $\mathcal{O}(t^{2D})$. Thus with a budget of $n$ function evaluations, the computational cost of implementing these algorithms is at least $\mathcal{O}\big(n^{2D+3} \big)$. The exponential dependence on $D$ makes the application of these algorithms to higher dimensions infeasible. Practical implementations of these algorithms in higher dimensions must employ certain heuristics and approximations, which do not come with theoretical guarantees.  In contrast, our algorithm admits a low complexity version with theoretical guarantees on the estimation error for which the cost of implementation has only a linear dependence on the dimension of the search space $\X$.  Thus for larger values of $D$, the cost of implementing the low complexity version of our algorithm can be significantly smaller than the state of the art.

In addition to the computational benefits, the convergence guarantees presented in Theorem~\ref{thm:alg1} for Algorithm~1 also   improve upon results of \citep{gotovos2013active} for the M\'atern family of kernels in two ways: 
\begin{itemize}
    \item The bounds provided by \citep{gotovos2013active} are only valid for $\nu>1$ since no explicit bounds on $\gamma_n$ are known for the M\'atern kernel with $\nu=1/2$. The dimension type bound of Theorem~\ref{thm:alg1}, in contrast, is valid for all $\nu \geq 1/2$. Thus by putting $\alpha = 1/2$ for the M\'atern 1/2 kernel, we obtain an explicit upper bound on the estimation error of the form $\mathcal{O}(n^{-1/(2D+2)})$ when $\X \subset [0,1]^D$.  
    
    \item  For the case of $\nu>1$, a sufficient condition under which the dimension type bounds given in (\ref{eq:dim_type}) are tighter than those of \citep{gotovos2013active} is when $D \geq \nu -1$. This implies that for the two most commonly used kernels in machine learning applications, M\'atern kernels with $\nu=3/2$ and $\nu=5/2$,  the bounds of Theorem~\ref{thm:alg1} are tighter than prior work for almost all dimensions. In Section~\ref{subsubsec:information_gain}, we will further relax this condition, by obtaining tighter bounds for all values of $\nu$ and $D$. 

\end{itemize}

 \subsubsection{Tighter bounds on Information Gain}
 \label{subsubsec:information_gain}

 As mentioned earlier, the standard approach of bounding the information gain of the $n$ evaluation points, as proposed by \cite{srinivas2012information},  is to first bound it with $\gamma_n$, and then use the explicit bounds on $\gamma_n$ derived in Theorem~5 of \cite{srinivas2012information}. This approach does not utlize any knowledge about the distribution of the evaluation points in the space $\X$. In the case of Algorithm~\ref{alg:bayesian_alg}, however,  since we know  the evaluation points are only selected from the set $\cup_{h \geq 0}\X_h$, we can use this to  provide a more fine grained characterization of the information gain. 

 \begin{theorem}
     \label{thm:information}
Suppose $Q_n$ denotes the times at which the algorithm performns function evaluations. 
Then 
for $J_n = \sum_{t\in Q_n} \sigma_t^{2}\lp \xhit \rp$, we have the following with probability at least 
$(1-\delta)$:
     \begin{equation}
         J_n \leq \sum_{h\geq 0}\big( \mathcal{I}_h(n_h, T_h) + \mathcal{O}(1)\big) 
     \end{equation}
     where $n_h$ is the number of function evaluations performed by Algorithm~\ref{alg:bayesian_alg} on points in $\X_h$, $T_h \in \{1,2,\ldots, n_h\}$ and 
         the term $\mathcal{I}_h(n_h, T_h)$ is defined as follows:
     \begin{multline}
         \mathcal{I}_h(n_h, T_h) = \max_{1 \leq s \leq n_h}\bigg( T_h\log( s m_h/\sigma^2) +\\ \sigma^{-2}(n_h -s)\sum_{i = T_h + 1}^{m_h}\hat{\lambda}_i \bigg).
     \end{multline}
     In the above display, $m_h = 2^h \lp \log \lp \frac{2^h h_{\max}}{\delta}\rp \rp$ and $\hat{\lambda}_i$ denotes the $i^{th}$ largest eigenvalue of the empirical covariance matrix computed at $m_h$ points 
     uniformly sampled from the set $\mc{Z}_h \coloneqq \bigcup_{x_{h,i} \in \X_h}B(x_{h,i}, \epsilon_h, 
     d)$ for $\epsilon_h = \min\{v_2\rho^h, 1/n_h\}$. 
 \end{theorem}
 \begin{proofoutline}
     The proof of this theorem proceeds similarly to the proof of Theorem~8 of \cite{srinivas2012information} by relating the information gain to the spectrum of the covariance matrix computed at some finite subset of $\X$ and then further approximating it the spectrum of the corresponding Hilbert-Schmidt operator associated with the covariance function. However, one key difference is that instead of computing the covariance matrix over a uniform grid over $\X$ (as in Lemma 7.7 of \citep{srinivas2012information}), we construct a sequence of uniform discretiztions by sampling points uniformly from sets of the form $\cup_{x \in \X_{h}}B(x,\epsilon_h, d)$ for all $h\geq 0$ and appropriate choice of $\epsilon_h$. Due to this,  we can replace the approximation error term (the last term in the statement of \citep[Theorem~8]{srinivas2012information}) with a $\mathcal{O}(1)$ term in the statement of our Theorem. The details are given in Appendix~\ref{appendix:information} 
 \end{proofoutline}

 We now instantiate the bound described in Theorem~\ref{thm:information} for the special case of Matern kernels with $\nu >1$.
 \begin{theorem}
     \label{thm:matern}
     Suppose $\X \subset [0,1]^D$ and      $I(y_{\mathcal{D}_n}; f_{\mathcal{D}_n})$ denotes the information gain for the set of points evaluated by Algorithm~\ref{alg:bayesian_alg}. Then if $f$ is sampled from $GP(0,k)$ where $k$ is a M\'atern kernel with smoothness parameter $\nu>1$, we have 
     \begin{equation}
         \label{eq:matern}
         I(y_{\mathcal{D}_n}; f_{\mc{D}_n}) = \tilde{\mathcal{O}}\big( n^{a} \big)
     \end{equation}
     where \[ a = \frac{D^2 + 3D}{4\nu + D^2 + 5D  }. \]
 \end{theorem}

 \begin{proofoutline}
     For proving the above theorem, we partition the evaluated points into two sets depending on whether their depth is more than some value $H\leq h_{max}$  or not. For the set of points with $h \leq H$, we bound the corresponding $\mathcal{I}_h$ values by making appropriate choice of the parameter $T_h$ which 
     balances the two terms of $\mc{I}_h$.  The term $n_h$ can be upper bounded by  
     a $\Oh\lp \rho^{-h(2\alpha + D)}\rp$ term, and $m_h$ can be upper bounded by a $\Oh\lp 2^h \log(n)\rp$
     term. 
 For the set of points with $h > H$ we use a bound on posterior standard deviation using Lemma~\ref{lemma:alg1} and the cell refining rule of Algorithm~\ref{alg:bayesian_alg}.
     Finally, the depth $H$ is chosen to balance the contributions of the terms 
     with $h\leq H$ and $h > H$. The details are given in Appendix~\ref{appendix:matern}. 
 \end{proofoutline}

     The bound given by the above theorem is tighter than the existing bound on $\gamma_n$ provided in Theorem~5 of \citep{srinivas2012information} for all values of $\nu >1$ and $D \geq 1$. Thus, in addition to the dimension dependent bound for $\nu=1/2$, by employing the above result we have obtained tighter characterization of the estimation error of our algorithm for all M\'atern kernels with half integer values of $\nu$ and for all values of $D$. 

     With some small modifications to the result of Theorem~\ref{thm:matern}, similar bounds on the information gain can be derived for the Gaussian Process bandit algorithms in \citep{shekhar2017gaussian}, thus proving tighter characterization of the cumulative regret for all M\'atern kernels.

\section{Conclusion and Future work}

In this paper we considered the problem of level set estimation of a black-box function from noisy observations. We proposed an algorithm for this problem in the Bayesian framework with GP prior and analyzed its performance. We showed that our proposed algorithm has lower computational complexity as well as tighter theoretical guarantees than existing algorithms. In the process, we also obtained tighter characterization of the information gain from $n$ function evaluations for our proposed algorithm. Finally, we also considered the problem of level set estimation in the non-Bayesian framework with certain smoothness assumptions, and proposed an algorithm which does not require the knowledge of the smoothness parameters. 

There are several directions along which the work presented in this paper can be extended. We conjecture that the bounds on the information gain of our algorithm obtained in Theorem~\ref{thm:matern} can be further improved by employing more careful counting arguments. Another important direction is to study the problem of level set estimation in the non-Bayesian setting, for example under H\"older continuity assumptions, and design computationally efficient algorithms which can automatically adapt to the unknown smoothness parameters.

\newpage

\subsection*{Acknowledgements}
The authors thank the three anonymous reviewers for their helpful feedback. 

\bibliographystyle{abbrvnat}
\bibliography{ref}

\newpage
\onecolumn
\begin{appendices}
\section{}

\subsection{Details of Algorithm~\ref{alg:bayesian_alg}}
\label{appendix:choice}

We now provide the detials of the parameters $\beta_n$ and $(V_h)_{h\geq0}$ of Algorithm~\ref{alg:bayesian_alg}. 
\paragraph{Choice of parameter $\beta_n$.}

Define the event $E_1 = \bigcap_{t \geq 1} E_{1,t}$ where we have $E_{1,t} \coloneqq \{ |f(x_{h,i}) - \mu_t(x_{h,i})| \leq \beta_n \sigma_t(x_{h,i}) \mid \hspace{0.5em} \forall x_{h,i} \in \X_t \forall t\geq 1\}$.

Suppose $t_n$ denotes the (random) time at which the algorithm performs its $n^{th}$
function evaluation. Then we have the following sequence of inequalities:
\begin{align*}
    Pr(E_1^c) &= \mbb{E}\lb \mbb{E}\lb \indi{E_1^c} \mid \X_t\rb \rb \\
              &\stackrel{(a)}{\leq} \mbb{E}\lb\sum_{t=1}^{t_n}  \mbb{E}\lb \indi{E_{1,t}^c} \mid \X_t\rb \rb 
    \stackrel{(b)}{\leq} \mbb{E} \lb \sum_{t=1}^{t_n} \sum_{x_{h,i} \in \X_t}2\exp\lp-\beta_n^2/2\rp \rb\\
    &\leq \mbb{E}\lb \sum_{t=1}^{t_n} 2|\X_t| \exp\lp -\beta_n^2/2\rp \rb 
    \stackrel{(c)}{\leq} \mbb{E}\lb 2t_n |\X_{h_{\max}}| \exp\lp -\beta_n^2/2 \rp \rb      \\ 
    & \stackrel{(d)}{\leq} \mbb{E} \lb 2n|\X_{h_{\max}}|^2 \exp \lp -\beta_n^2/2 \rp \rb = 2n|\X_{h_{\max}}|^2 \exp \lp -\beta_n^2/2 \rp. \\
\end{align*}
In the above display, $(a)$ follows from union bound, $(b)$ uses the Gaussian tail 
inequality, $(c)$ employs the fact that all realizations of $\X_t$ must have 
cardinality smaller than or equal to $\X_{h_{\max}}$, and $(d)$ follows from the fact that $t_n$
must be smaller than $|\X_{h_{\max}}| + n \leq |\X_{h_{\max}}|n$. 
Thus for restricting the probability of $E_1^c$ to less than $\delta$, an appropriate 
choice of $\beta_n$ is $\sqrt{2 \log(2n2^{2h_{max}}/\delta)} = \sqrt{2 \log\lp 2 
n^{1 + \frac{2}{2\alpha \log(1/\rho)}}\rp + 2 \log(1/\delta)}$.

\paragraph{Choice of the parameter $V_h$.}

For describing the choice of $V_h$, we define the event $E_2 = \bigcap_{0 \leq h \leq 
h_{\max}}\bigcap_{1 \leq i \leq 2^h}E_2^{(h,i)}$, where we have $E_2^{(h,i)}
\coloneqq \{ \sup_{x_1,
 x_2 \in \X_{h,i}} |f(x_1)-f(x_2)| \leq V_h 
\hspace{0.4em} \forall 1 \leq i \leq 2^h \}$. Then, we have the following 
\begin{align*}
    Pr\lp  E_2^c\rp &\leq \sum_{h=0}^{h_{max}} \sum_{i=1}^{2^h} Pr\bigg( \sup_{x_1, x_2 \in \X_{h,i}} |f(x_1) - f(x_2)| \leq V_h \bigg) \\
\end{align*}

Now, by assumptions \textbf{C1} and \textbf{X3}, we know that $\X_{h,i} \subset 
B(x_{h,i}, g(v_1\rho^h), d_k)$, and thus we have 
\begin{align*}
    & Pr\lp \sup_{x_1, x_2 \in \X_{h,i}}|f(x_1) - f(x_2)| > V_h  \rp \\ 
\leq & Pr\lp  \sup_{ x_1, x_2 \in B(x_{h,i}, g(v_1\rho^h), d_k) }|f(x_1) - f(x_2)| > V_h \rp \\
\coloneqq & Pr(F_{h,i}^c)                                                                      
\end{align*}

Now, using Proposition~1 of \citep{shekhar2017gaussian}, we have for $V_h = g(v_1\rho^h)\lp \sqrt{
C_2 + 2\log(1/\delta_h) + (4D_m')\log(1/v_1\rho^h) } + C_3\rp$, $Pr(F_{h,i}^c) \leq \delta_h$.
In obtaining  this expression, we have used the fact that the metric dimenison of 
$(\X, d_k)$, denoted by $D_m'$,  is finite (see Lemma~\ref{lemma:metric_dim2} 
at the end of this section for details). 
The constant $C_2$ is equal to $2\log \lp 2C_{2D_m'}'\pi^2/6 \rp$ where $C_{2D_m'}'$
is the leading constant corresponding to the exponent $2D_m'$ for computing the 
bounds on the covering numbers of $(\X, d_k)$. The term $C_3$ is equal to 
$\lp \sum_{n\geq 1}2^{-(n-1)} \sqrt{\log n} \rp + \lp \sum_{n \geq 1}2^{-(n-1)}
\sqrt{n 2D_m' \log(2)}\rp$.

Finally, with this choice of $V_h$ and  with $\delta_h = \delta/(2^h h_{\max})$ 
for all values of $0 \leq h \leq h_{\max}$,   we get that $Pr(E_2^c) \leq \delta$.

\vspace{1em}

    \begin{remark}
        \label{remark:V_h}
        Without loss of generality, we can assume that the sequence $(V_h)_{h\geq0}$ is such that for all $h \geq 0$, we have $V_h \leq 2V_{h+1}$. This is because 
        \begin{align*}
            \sup_{x_1, x_2 \in \X_{h,i}}|f(x_1) - f(x_2)| &\leq \sup_{x_1, x_2 \in \X_{h,i}}|f(x_1) - f(z_1)| + f(x_2) - f(z_2)|  + |f(z_1) - f(z_2)|\\
    \end{align*}
    for any $z_1, z_2$ by triangle inequality. If we select $z_1 \in \X_{h+1, 2i-1}$ and $z_2 \in \X_{h+1, 2i}$ and $d(z_1, z_2) \leq \epsilon $ for arbitrary $\epsilon>0$, we get that  
    \[
    \sup_{x_1, x_2 \in \X_{h,i}} |f(x_1) - f(x_2)| \leq V_{h+1} + V_{h+1} + \epsilon
    \]
    Thus $2V_{h+1}$ is a valid upper bound on $\sup_{x_1, x_2 \in \X_{h,i}}|f(x_1) - f(x_2)|$, and given any sequence of $(V_h)_{h \geq 0}$ we can replace $V_h \leftarrow \min \{ V_h, 2 V_{h+1}\}$ for $h = h_{max}, h_{max}-1, \ldots, 0$ to impose the condition.  

    \end{remark}

Finally, we end this section by stating and proving the result about the metric 
dimension of the space $(\X, d_k)$. 

\begin{lemma}
    \label{lemma:metric_dim2}
    Suppose the metric space $(\X, d)$ has a finite metric dimension $D_m$ and 
    suppose $k$ is a covariance function satisfying the conditions \textbf{C1} and 
    \textbf{C2}. Then, the metric dimension of $(\X, d_k)$, denoted by $D_m'$ is 
    upper bounded by $D_m/\alpha$. 
\end{lemma}

\begin{proof}
    By the assumption of the finite metric dimension of $(\X, d)$, we know that 
    for any $a>D_m$, there exists a constant $0<C_a < \infty$ such that for all 
    $r>0$, we have $N(\X,r, d) \leq C_a r^{-a}$. Now, for an $r>0$, consider the 
    packing number $N(\X, r, d_k)$. We have the following two cases:
    \begin{itemize}
        \item If $r< C_k\delta_k^{\alpha}$ then we claim that $N(\X, r, d_k) \leq
            N\lp \X, \lp \frac{r}{C_k}\rp^{1/\alpha}, d\rp \leq C_a \lp \frac{r}
            {C_k}\rp^{-a/\alpha}$. To see this, let $\mc{C}$ 
            denote  any $\lp \frac{r}{C_k}\rp^{1/\alpha}$-covering set of 
            $(\X, d)$. Then, by definition,  for any $x \in \X$, there exists a $z
            \in \mc{C}$ such that $d(x,z) \leq \lp \frac{r}{C_k}\rp^{1/\alpha}$, 
            which by the assumption \textbf{C1} implies that $d_k(x,z) \leq r$. 
            This implies that $\mc{C}$ is an $r$-covering set for $(\X, d_k)$. 
            Thus we conclude that for any $a>D_m$, we have $N(\X, r, d_k) \leq 
            C_a C_k^{a/\alpha} r^{-a/\alpha}$ for all $r \leq C_k\delta_k^{\alpha}$. 

        \item For the case of $r>C_k\delta_k^{\alpha}$ we use the fact that the packing 
            number $N(\X, r, d_k)$ is monotonically nonincreasing in $r$, and thus
            for such values of $r$, we have $N(\X, r, d_k) \leq N(\X, C_k\delta_k^
            {\alpha}, d_k)$. The term $N(\X, C_k\delta_k^{\alpha}, d_k)$ can be 
            upper bounded by $C_a\delta_k^{-a}$ for all $a>D_m$, which implies 
            that $N(\X, r, d_k) \leq C_a \lp \frac{\text{diam}(\X)}{\delta_k^{\alpha}}\rp
            ^{a/\alpha}r^{-a/\alpha}$.
    \end{itemize}
    Combining the above two observations, we see that for all $a >D_m$, we have $N(
    \X, r, d_k) \leq C_a' r^{-a/\alpha}$ for all $r>0$, where we can choose $C_a'
    = C_a \lp C_k^{a/\alpha} + \frac{\text{diam}(\X)}{\delta_k^{\alpha}} \rp^{a/\alpha}
    < \infty$.
    This implies that the metric dimension of $(\X,d_k)$ can be no larger than  
    $D_m/\alpha$. 

\end{proof}

\newpage
\subsection{Proof of Theorem~\ref{thm:alg1}}
\label{appendix:convergence}
Throughout this proof, we make all the arguments under the assumption that the events $E_1$ and $E_2$ defined in Appendix~\ref{appendix:choice} hold true. 
We introduce the notation $A_t = \max \{ \bar{u}_t(x_{h,i}) - \tau, \tau - \bar{l}_t(x_{h,i})\}$ for the index used in selecting the point $x_{h_t, i_t}$. Then similar to \citep{gotovos2013active}, we first observe that by construction, the term $A_t(x_{\bar{h}_t, \bar{i}_t})$ is non-increasing in $t$. 
Furthermore, since for all $t$, the set $\hat{S}_t \subset S_{\tau}$, for all $x \in \X \setminus (\hat{S}_t \cup \hat{R}_t)$, we have with high probability \[\LL(\hat{S}_t, S_{\tau}) \leq \sup_{x \in \X \setminus (\hat{S}_t \cup \hat{R}_t)}|f(x) - \tau| \leq A_t. \]
The inequality in the above display follows from the fact that with high probability, 
at any time $t$, $\hat{S}_t \subset S_{\tau}$ and $\hat{R}_t \subset S_\tau^c$. Thus 
the ambiguous region is $\X \setminus \lp \hat{S}_t \cup \hat{R}_t\rp$ and the term 
$\mc{L}\lp \hat{S}_t, S_\tau\rp$ can be upper bounded by the maximum possible deviation
from $\tau$ for points in this region. 
Thus at the end of $n$ function evaluations, using the monotonicity of $A_t$,  we have 
\[
    \LL(\hat{S}_n, S_{\tau}) \leq A_n(x_{h_{t_n}, i_{t_n}}) \leq \frac{1}{n} \sum_{j=1}^n A_{t_j}\lp x_{x_{\bar{h}_{t_j}, \bar{i}_{t_j}}}\rp
\]
where $t_j$ denote the time at which the $j^{th}$ function evaluation is performed by the algorithm. 

From Lemma~\ref{lemma:alg1}, we can upper bound the above in two ways
\begin{align*}
    \LL_t(\hat{S}_{\tau}, S_{\tau}) &\leq \frac{4\beta_n}{n} \sum_{t \in Q_n} \sigma_t(x_{h_t,i_t})\\
    \LL_t(\hat{S}_{\tau}, S_{\tau}) & \leq \min_{t \in Q_n} 10V_{h_t} 
\end{align*}
where $Q_n = \{t_1, t_2, \ldots, t_n\}$ represents the set of times $t$ at which Algorithm~\ref{alg:bayesian_alg} performs function evaluations. 

To obtain the dimension type bound (\ref{eq:dim_type}), it suffices to obtain a lower bound on the largest value $h_t$ for $t \in Q_n$, which gives us an upper bound on $\min_{t \in Q_n} V_{h_t}$. 

We proceed according to the arguments used in \cite{munos2011soo}. We have 
\begin{align*}
    n = \sum_{h=0}^{h_{max}}n_h
\end{align*}
where $n_h$ is the number of times the algorithm evaluated points in $\X_h$.
We now observe that  $n_h \leq q_h | \X_h \cap \mc{W}_h|$, where we use the notation $\mc{W}_h 
= \{x \in \X \mid \, |f(x)-\tau| \leq 10V_h\}$ and $q_h$ is the upper bound on the 
number of times a point at level $h$ of the tree is evaluated by the algorithm. 
By assumption \textbf{X3}, we know that the points in $\X_{h}$ are at least $2v_2\rho^{h}$
separated from each other. Thus we can bound $|\X_h \bigcap \mc{W}_h|$ with the 
packing number $M\lp \mc{W}_h, 2v_2\rho^{h}, d\rp$, which by the definition of the 
dimension term $\tilde{D}$ can be further upper bounded by $\mc{O}\lp \rho^{-h \tilde{D}} \rp$. 
Next, we define $h_0$ to be the largest depth such that we have 
\begin{align*}
    n & \geq \sum_{h=0}^{h_0} q_h \mathcal{O}(\rho^{-h\tilde{D}}) 
\end{align*}

Now, using the value on $q_h = \Oh(\rho^{-2h\alpha})$ given in Lemma~1, we can conclude that a suitable value of $h_0$ is $\log(n)/(\tilde{D} + 2\alpha)\log(1/\rho)$ (Here we assumed that $n$ is large enough so that $v_2\rho^{h_0} \leq \delta_k$). Finally, the fact that $\max_{t \in Q_n} h_t \geq h_0$, we have that 
\[
\LL_t(\hat{S}_{\tau}, S_{\tau}) \leq 10V_{h_0} = \tilde{\Oh}(n^{-\alpha/(\tilde{D} + 2\alpha)})
\]
where the equality follows by plugging in the value of $h_0$ in the expression for $V_h$. 
For obtaining the information type bounds, we partition the set $Q_n$ into $Q_{n,1} \cup Q_{n,2}$, where $Q_{n,1} = \{ t \in Q_n \mid h_t < h_{max} \}$ and $Q_{n,2} = Q_n \setminus Q_{n,1}$.
Then we have 
\[
    \LL(\hat{S}_{\tau}, S_{\tau}) \leq \frac{1}{n} \bigg(\sum_{t \in Q_{n,1}}4\beta_n \sigma_t(x_{h_t,i_t}) + \sum_{t \in Q_{n,2}} 10 V_{h_{max}}\bigg)
\]

Now, using the fact that $h_{max} \geq \frac{\log n}{2\alpha \log(1/\rho)}$ we get that the second term on the right side above is $\tilde{\Oh}({1/\sqrt{n}})$
In the first term, we can now use Cauchy-Schwarz inequality to upper bound $\sum_{t \in Q_{n,1}}\sigma_t(x_{h_t,i_t}) \leq \sqrt{|Q_{n,1}| \sum_{t \in Q_{n,1}}\sigma_t^2(x_{h_t,i_t})}$. Combining these two results, 
we get 
\[
    \LL( \hat{S}_{\tau}, S_{\tau}) \leq \frac{1}{n} \bigg( 4 \beta_n\sqrt{n J_n}  + \tilde{\mathcal{O}}(\sqrt{n})\bigg)
\]
which gives us the required inequality (\ref{eq:inf_type1}).

\subsection{Proof of Theorem~\ref{thm:information}}
\label{appendix:information}

Recall that $J_n = \sum_{t \in Q_n} \sigma_t\lp x_{h_t,i_t}^2\rp$, where $Q_n$ is 
the set of times at which the algorithm performs a function evaluation. By introducing 
the notation $Q_{n,h} \coloneqq \{t \in Q_n \mid \, h_t = h\}$, we can rewrite $J_n$
as $J_n = \sum_{h=0}^{h_{\max}} \sum_{t \in Q_{n,h}} \sigma_t^2\lp x_{h_t,i_t}\rp
\coloneqq \sum_{h=0}^{h_{\max}}J_{n,h}$, where the term $J_{n,h}$ is defined 
implicitly.

Next we consider the term $\Jnh$. Introduce the notation $\mc{D}_{j} = \{ (\xhit, y_t)
    \mid t \in Q_n, \, t < t_j \}$ and $\mc{D}_{j, h} = \{ (\xhit, y_t) \mid t \in 
Q_{n,h}, \, t < t_j\}$. Then for any $t = t_j \in Q_{n,h}$, we have $\sigma_{t}^2\lp \xhit 
\rp = \text{var}\lp f\lp \xhit\rp \mid \mc{D}_j \rp$, which can be upper bounded by 
$\tilde{\sigma}_t^2\lp \xhit \rp \coloneqq \text{var}\lp f\lp \xhit \rp \mid \mc{
D}_{j, h} \rp$. This follows from the observation that the posterior variance at 
any point conditioned on $\mc{D}_{j,h}$ must be greater than or equal to the 
posterior variance conditioned on $\mc{D}_j$. The proof of this statement follows 
from the first part of Proposition~3 of \citep{shekhar2017gaussian}. Thus we have
$J_n \leq \sum_{h=0}^{h_{\max}}\Jnht$.

Since $\sigma^{-2} \sigma_t^2\lp \xhit \rp \leq \frac{\sigma^{-2}}{\log(1 + \sigma^{-2})}
\log \lp 1 + \sigma^{-2}\sigma_t^2\lp \xhit \rp \rp \coloneqq C_4 \log \lp 1 + 
\sigma^{-2}\sigma_t^2\lp \xhit \rp \rp$, we get that $\Jnht \leq C_4 I\lp
f_{S_{n,h}}; y_{S_{n,h}}\rp$ where $S_{n,h} = \{\xhit \mid \, t \in Q_{n,h}\}$. 

We now define $\mc{Z}_h = \bigcup_{x \in \X_h}B(x, \epsilon_h, d)$ with $\epsilon_h 
= \min\{ v_2\rho^h, 1/n_h \}$. Since we have $S_{n,h} \subset \mc{Z}_h$, we can 
further upper bound $\Jnht$ with the term $\gamma_n^{(h)} \coloneqq \max_{S \subset
\mc{Z}_h \,, |S| = n_h} I\lp f_S ; y_S \rp$. Let $\hat{\X}_h$ represent
the set consisting of $m_h$ samples drawn uniformly from the set $\mathcal{Z}_h$. 

The rest of the proof follows the steps in the proof of Theorem~5 of \citep{srinivas2012information}
with some modifications. 
\begin{itemize}

    \item For a given $\delta>0$, for $m_h = 2^h \log \lp \frac{2^h h_{\max}}{\delta}\rp$
        with probablitiy at 
        least $1-\delta/h_{\max}$, every  point $x_{h,i} \in \X_h$ has an $\epsilon_h$ neighbor 
        in the set $\hat{\X}_h$. 
        \begin{proof}
            Consider a fixed point $x_{h,i} \in \X_h$ .Then the probability that 
            a point sampled uniformly over $\mc{Z}_h$ does not lie in the $B(x_{h,i},
            \epsilon_h, d)$ is $1-1/2^h$. Then the probability that no point in $
            \hat{\X}_h$ lies in $B(x_{h,i}, \epsilon_h, d)$ is equal to $\lp 1-
            2^{-h}\rp^{m_h}$. Finally, by union bound over the elements of $\X_h$, 
            we get that that the probability that there exists a point $x_{h,i} \in
            \X$ with no $\epsilon_h$ neighbor in $\hat{\X}_h$ is upper bounded by
            $2^h\lp 1 - 2^{-h}\rp^{m_h} \leq 2^h \exp\lp -2^h m_h \rp$. Setting 
            this equal to $\delta/h_{\max}$ gives us the required value of $m_h$. 

        \end{proof}

    \item Similar to \citep{srinivas2012information}, a restricted version of the 
        maximum information gain can be defined as $\ght \coloneqq \max_{S \subset \hat{\X}_h,\,
        |S|=n_h} I\lp f_S; y_S \rp$. Using the Lipschitz property of information 
        gain \citep[Lemma~7.4]{srinivas2012information} and the fact that with high
        probabilty $\hat{\X}_h$ contains $\epsilon_h$ neighbors of every point in 
        $\X_h$, we can conclude that $\gh \leq \ght + n_h \epsilon_h = \ght + \Oh(1)$
         since $\epsilon_h \leq 1/n_h$. 

     \item Finally, by application of Lemma~7.8 of \citep{srinivas2012information}, 
         we can get an upper bound on $\ght$ as follows for any $T_h \in \{1,2,\ldots, n_h\}$:
         \[
             \ght \leq \max_{1 \leq s \leq n_h}\lp T_h \log\lp s m_h/\sigma^2\rp 
             + \lp n_h - s \rp \sigma^{-2} \sum_{i=T_h+1}^{m_h} \hat{\lambda}_i\rp 
             \coloneqq \mc{I}_h\lp n_h, T_h\rp. 
         \]

\end{itemize}

To conclude the proof, we combine the above results to observe that the following 
sequence of inequlities hold with probability at least $(1-\delta)$,
\begin{align*}
    J_n &= \sum_{h=0}^{h_{\max}}\Jnh \leq \sum_{h=0}^{h_{\max}}\Jnht \\
        & \leq C_4 \sum_{h=0}^{h_{\max}} I\lp f_{S_{n,h}}, y_{S_{n,h}} \rp \\
        &\leq C_4\sum_{h=0}^{h_{\max}}\gh \leq C_4 \sum_{h=0}^{h_{\max}}\lp  \ght + \Oh(1) \rp \\
        & \leq C_4 \sum_{h=0}^{h_{\max}} \lp \mc{I}_h\lp n_h, T_h \rp + \Oh(1) \rp. 
\end{align*}


\subsection{Proof of Theorem~\ref{thm:matern}}
\label{appendix:matern}

We first note that to upper bound $I\lp y_{\mc{D}_n}; f_{\mc{D}_n}\rp$, it suffices 
to get an upper bound on $J_n$. This is because $I\lp y_{\mc{D}_n}; f_{\mc{D}_n}\rp
= \sum_{t \in Q_n} \log\lp 1 + \sigma^{-2}\sigma_t^2\lp \xhit \rp \rp \stackrel{(a)}{\leq} \sum_{t
\in Q_n}\sigma^{-2}\sigma_t^2\lp \xhit \rp = \sigma^{-2}J_n$, where $(a)$ uses the
fact that for  all $z\geq 0$, we have $\log(1 + z) \leq z$. 

To upper bound $J_n$, we first partition the evaluated points into two sets depending on whether their depth is smaller or 
larger than some value $H \leq h_{max}$ (to be decided later). For points with $h \leq H$, we proceed as follows:
\begin{itemize}

    \item Using Lemma~7.7 of \citep{srinivas2012information}, for all 
        $h \geq 0$, we can  select $\hat{\X}_h$ such that the following inequality holds:
        \[
            \sum_{i=T_h+1}^{m_h} \hat{\lambda}_i \leq m_h \lp \sum_{i \geq T_h+1}
                \lambda_i +\frac{\delta}{h_{\max}} \rp,  
            \]
            where $\lambda_i$ is the $i^{th}$ largest eigenvalue of the Hilbert-
            Schmidt operator associated with the kernel $k$ and the uniform measure 
            on the set $\mc{Z}_h$. 
        \item On simplification, we get 
            \begin{equation}
                \label{eq:thm3_1}
                \mathcal{I}_h(n_h, T_h) \leq \max_{s} \bigg( T_h \log( s m_h/\sigma^2) +  (1-s/n_h)m_hn_h\big( \sum_{i \geq T_h+1} \lambda_i \big) \bigg)  + m_h n_h \frac{\delta}{h_{\max}}. 
            \end{equation}

            By selecting $\delta= n^{-2D}$ and using the fact that $h\leq h_{\max}$, 
            the last term of \eqref{eq:thm3_1} is $\mc{O}\lp \log n \rp$. 

            For the sequel, we focus on the first term of \eqref{eq:thm3_1}. We set 
            $s = n_h$ in the first part of the first term of \eqref{eq:thm3_1} and
            use the fact that $1-s/n_h \leq 1$ for all choices of $s$ to get the 
            following upper bound. 
        \begin{equation}
            \label{eq:thm3_2}
            \mc{I}_h\lp n_h , T_h \rp \leq T_h \log \lp sm_h/\sigma^2 \rp + m_h n_h R(T_h) + \Oh(\log n), 
    \end{equation}
    where we have $R(T_h) = \sum_{i \geq T_h + 1} \lambda_i$ denotes of the tail 
    sum of the eigenvalues $(\lambda_i)_{i \geq 0}$. 

For the operator associated with the M\'atern kernel with
smoothness parameter $\nu$, we have $R(T_h)  = \Oh(T_h^{1-(2\nu+D)/D})$ \citep{seeger2008information,
srinivas2012information}. Furthermore, we have $m_h = 2^h \log (2^h h_{\max}/\delta) 
\leq 2^h\lp  4D \log(n)\rp$, and $n_h = \Oh\lp \rho^{-h(\tilde{D} + 2\alpha)} \rp$
as derived in the proof of Theorem~\ref{thm:alg1}. 
    
An appropriate choice of $T_h$ is $ \lp m_h n_h \rp^{D/(D+2\nu)}$ which balances 
the two components up to logarithmic factors. Thus, by pluggin this value of $T_h$, 
we get the following upper bound on $\mc{I}_h\lp n_h, T_h\rp$
\begin{equation}
    \label{eq:thm3_3}
    \mc{I}_h \lp n_h, T_h \rp = \tilde{\Oh}\lp \lp m_h n_h \rp^{D/(D + 2\nu} \rp
    = \tilde{\Oh}\lp \lp \frac{1}{\bar{\rho}}\rp^{\frac{h(D + 2\alpha + 1)D}{D + 2\nu}} \rp,  
\end{equation}
where $\bar{\rho} = \min\{\rho, 1/2\}$ and $\tilde{\Oh}$ hides polylogarithmic 
factors. 
        \item With this choice of $T_h$ for all $h \leq H$, and  summing these terms for $ 0 \leq h \leq H$, we get an upper bound of the form
            \begin{equation}
                \label{eq:thm3_4}
                \sum_{h=0}^{H} \mc{I}_h\lp n_h, T_h \rp = \tilde{\Oh}\lp \lp \frac{1}{\bar{\rho}}\rp^{\frac{H(D+1 + 2\alpha)D}{D + 2\nu}}\rp.
        \end{equation}
\end{itemize}

For the evaluated points $x_{h,i}$ with $h > H$, we use the fact that $\sigma_t(x_{h,i}) \leq \frac{V_H}{\beta_h}$ by the rule used by Algorithm~\ref{alg:bayesian_alg} for refining cells at level $H$. The number of such evaluations can be trivially upper bounded by $n$, thus providing an upper bound on the contribution of such evaluations to $J_n$ of the form $\tilde{\Oh}(n\rho^{2H})$ where we used the fact that
$V_H = \tilde{\Oh}\lp \rho^{H \alpha} \rp$ and 
$\alpha=1$ for M\'atern kernels with $\nu>1$.

Thus by balancing the two contributions, an appropriate choice of $H$ is given by 
\[
    H = \lp \frac{2 \nu + D}{4\nu + D^2 + 5D} \rp \frac{\log n}{\log\lp 1/\bar{\rho}\rp}
\] which is smaller than $h_{max}$ for all values of $\nu$ and $D$. Thus, with this choice of $H$, we get the required bound of Theorem~\ref{thm:matern}.

\end{appendices}

\end{document}